\makeatletter

\PassOptionsToPackage{svgnames}{xcolor}

\PassOptionsToPackage{noend}{algorithmic}
\PassOptionsToPackage{noend}{algpseudocode}

\AddToHook{package/lineno/before}{\RequirePackage{amsmath}}

\AtBeginDocument{
  \@ifpackageloaded{theapa}{
    \NewDocumentCommand{\citet}{o m}{
      \IfNoValueTF{#1}
        {\citeauthor{#2} (\citeyear{#2})}
        {\citeauthor{#2} (\citeyear[#1]{#2})}
    }
    \NewDocumentCommand{\citep}{o m}{
      \IfNoValueTF{#1}
        {\cite{#2}}
        {\cite[#1]{#2}}
    }
  }{
  }
}

\AddToHook{package/algorithmic/after}{
  \renewcommand{\algorithmiccomment}[1]{\hfill \# #1}
  \def\State\STATE
  \def\If\IF
  \def\Then\THEN
  \def\Elsif\ELSIF
  \def\Else\ELSE
  \def\Endif\ENDIF
  \def\For\FOR
  \def\Forall\FORALL
  \def\Do\DO
  \def\Endfor\ENDFOR
  \def\While\WHILE
  \def\Endwhile\ENDWHILE
  \def\Repeat\REPEAT
  \def\Until\UNTIL
  \def\Return\RETURN
  \def\Require\REQUIRE
  \def\Ensure\ENSURE
  \def\Comment\COMMENT
}

\AddToHook{package/algpseudocode/after}{
  \algrenewcommand\algorithmicindent{0.7em}
  \algrenewcommand{\algorithmiccomment}[1]{\hfill \# #1}
  \algblockdefx[method]{Method}{EndMethod}[2]{\textbf{method} \textsc{#1} (#2)}{\algorithmicend}
  \ifthenelse{\equal{\ALG@noend}{t}}
   {
   \algtext*{EndMethod}
   }{}
  \algblockdefx[try]{Try}{EndTry}[0]{\textbf{try}}{\algorithmicend}
  \ifthenelse{\equal{\ALG@noend}{t}}
   {
   \algtext*{EndTry}
   }{}
  \algblockdefx[except]{Except}{EndExcept}[2]{\textbf{except} \textsc{#1} \textbf{as} #2}{\algorithmicend}
  \ifthenelse{\equal{\ALG@noend}{t}}
   {
   \algtext*{EndExcept}
   }{}
  \algblockdefx[finally]{Finally}{EndFinally}[0]{\textbf{finally}}{\algorithmicend}
  \ifthenelse{\equal{\ALG@noend}{t}}
   {
   \algtext*{EndFinally}
   }{}
  \def\STATE\State
  \def\IF\If
  \def\THEN\Then
  \def\ELSIF\ElsIf
  \def\ELSE\Else
  \def\ENDIF\EndIf
  \def\FOR\For
  \def\FORALL\ForAll
  \def\DO\Do
  \def\ENDFOR\EndFor
  \def\WHILE\While
  \def\ENDWHILE\EndWhile
  \def\REPEAT\Repeat
  \def\UNTIL\Until
  \def\RETURN\Return
  \def\REQUIRE\Require
  \def\ENSURE\Ensure
  \def\COMMENT\Comment
}

\makeatother

\documentclass[letterpaper]{article}
\usepackage{aaai26}
\usepackage{times}
\usepackage{helvet}
\usepackage{courier}
\usepackage[hyphens]{url}
\usepackage{graphicx} 
\usepackage{caption} 
\usepackage{natbib}  
\makeatletter
\@ifpackageloaded{theapa}{
}{
 \usepackage{natbib}
}
\makeatother

\usepackage{xparse}

\usepackage{amsmath}
\usepackage{amsthm}
\usepackage{amssymb}
\usepackage{amsfonts}

\usepackage{tabularx}
\usepackage{booktabs}   
\usepackage{multirow}
\usepackage{subfigure}

\usepackage{manfnt}             
\usepackage{cancel}
\usepackage{soul}

\usepackage{xspace}             
\usepackage{relsize}            
\usepackage{adjustbox}          
\usepackage{diagbox}            
\usepackage{pdflscape}          

\usepackage{microtype}

\usepackage{graphicx}
\usepackage{url}
\usepackage{comment}            
\usepackage{xcolor}             
\usepackage{colortbl}

\usepackage{algorithm}
\usepackage{algorithmic} 

\usepackage{xr}

\usepackage{bm}
\usepackage{alphabeta}          
\usepackage{stmaryrd}

\def\eqref#1{equation~\ref{#1}}

\def\1{\bm{1}}
\def\0{\bm{0}}

\def\rh{{\textnormal{h}}}

\def\rr{{\textnormal{r}}}

\def\rx{{\textnormal{x}}}
\def\ry{{\textnormal{y}}}

\DeclareMathAlphabet{\mathsfit}{\encodingdefault}{\sfdefault}{m}{sl}
\SetMathAlphabet{\mathsfit}{bold}{\encodingdefault}{\sfdefault}{bx}{n}

\def\E{{\mathbb{E}}}

\def\N{{\mathcal{N}}}

\def\R{{\mathbb{R}}}

\def\Z{{\mathbb{Z}}}

\newcommand{\Var}{\mathrm{Var}}

\DeclareMathOperator*{\argmin}{arg\,min}

\newcommand{\brackets}[1]{{\left<#1\right>}}
\newcommand{\braces}[1]{{\left\{#1\right\}}}
\newcommand{\parens}[1]{{\left(#1\right)}}

\newcommand{\dbrackets}[1]{{\left\llbracket#1\right\rrbracket}}

\newcommand{\then}{\therefore \qquad}

\NewDocumentCommand{\diffby}{s m O{}}{
 \IfBooleanTF{#1}
  {\frac{\partial#3}{\partial#2}}
  {\frac{d#3}{d#2}}
}

\newcommand{\satisfies}{\vDash}

\RenewDocumentCommand{\to}{o o}{
 \IfNoValueTF{#1}
  {\rightarrow}
  {\IfNoValueTF{#2}
   {\xrightarrow{#1}}
   {\xrightarrow[#2]{#1}}}
}

\NewDocumentCommand{\affect}{o o}{
 \IfNoValueTF{#1}
  {\rightsquigarrow}
  {
   \IfNoValueTF{#2}
   {\rightsquigarrow^{#1}}
   {\rightsquigarrow^{#1}_{#2}}
  }
}

\renewcommand{\then}{\Rightarrow}

\newcommand{\iid}{i.i.d.\xspace}

\newcommand{\gp}{\mathrm{GP}}

\makeatletter

\usepackage{pdftexcmds}
\usepackage{etoolbox}
\newtoggle{dev}
\togglefalse{dev}

\ifcase\pdf@shellescape
  \or
  \toggletrue{dev}
  \or
\fi

\newcommand{\todo}[1]{\iftoggle{dev}{\red{\textbf{#1}}}{}}

\usepackage{stackengine}
\stackMath
\newcommand\tsup[2][2]{
 \def\useanchorwidth{T}
  \ifnum#1>1
    \stackon[-.5pt]{\tsup[\numexpr#1-1\relax]{#2}}{\scriptscriptstyle\sim}
  \else
    \stackon[.5pt]{#2}{\scriptscriptstyle\sim}
  \fi
}

\newcommand{\function}[1]{\textsc{#1}}

\newcommand{\mycolor}[2]{\textcolor{#1}{#2}}

\newcommand{\red}[1]{\mycolor{red}{#1}}

\def\_{\\[-0.3em]}

\newlength{\maxwidth}

\newtheorem{defi}{Definition}

\newtheorem{theo}{Theorem}

\let\@myref\ref

\renewcommand{\ref}[1]{\@myref{#1}\iftoggle{dev}{\todo{(Do not use ``ref'' directly!)}}{}}

\newcommand{\refsec}[1]{Sec.\,\@myref{#1}}
\newcommand{\refseq}[1]{Sec.\,\@myref{#1}}
\newcommand{\refig}[1]{Fig.\,\@myref{#1}}
\newcommand{\reftbl}[1]{Table \@myref{#1}}
\newcommand{\refstep}[1]{Step \@myref{#1}}
\newcommand{\refalgo}[1]{Alg.\,\@myref{#1}}
\newcommand{\refchap}[1]{Chap.\,\@myref{#1}}
\newcommand{\reflst}[1]{List \@myref{#1}}
\newcommand{\refeq}[1]{Eq.\,\@myref{#1}} 

\newcommand{\reftheo}[1]{Thm.\,\@myref{#1}}
\newcommand{\refline}[1]{line\,\@myref{#1}}
\newcommand{\refdef}[1]{Def.\, \@myref{#1}}
\newcommand{\refex}[1]{Example\,\@myref{#1}}
\newcommand{\refconv}[1]{Conv.\,\@myref{#1}}
\newcommand{\reffact}[1]{Fact.\,\@myref{#1}}
\newcommand{\reflemma}[1]{Lemma.\,\@myref{#1}}
\newcommand{\refcorol}[1]{Col.\,\@myref{#1}}

\newcommand{\refsecs}[2]{Sec.\,\@myref{#1}-\@myref{#2}}
\newcommand{\refseqs}[2]{Sec.\,\@myref{#1}-\@myref{#2}}
\newcommand{\refigs}[2]{Fig.\,\@myref{#1}-\@myref{#2}}
\newcommand{\reftbls}[2]{Tables \@myref{#1}-\@myref{#2}}
\newcommand{\refsteps}[2]{Steps \@myref{#1}-\@myref{#2}}
\newcommand{\refalgos}[2]{Alg.\,\@myref{#1}-\@myref{#2}}
\newcommand{\refchaps}[2]{Chap.\,\@myref{#1}-\@myref{#2}}
\newcommand{\reflsts}[2]{Lists \@myref{#1}-\@myref{#2}}
\newcommand{\refeqs}[2]{Eq.\,\@myref{#1}-\@myref{#2}}
\newcommand{\refpages}[2]{p.\pageref{#1}-\@myref{#2}}
\newcommand{\reftheos}[2]{Thm.\,\@myref{#1}-\@myref{#2}}
\newcommand{\reflines}[2]{line\,\@myref{#1}-\@myref{#2}}
\newcommand{\refdefs}[2]{Def.\,\@myref{#1}-\@myref{#2}}
\newcommand{\refexs}[2]{Example\,\@myref{#1}-\@myref{#2}}
\newcommand{\refconvs}[2]{Conv.\,\@myref{#1}-\@myref{#2}}
\newcommand{\reffacts}[2]{Facts.\,\@myref{#1}-\@myref{#2}}
\newcommand{\reflemmas}[2]{Lemma.\,\@myref{#1}-\@myref{#2}}
\newcommand{\refcorols}[2]{Col.\,\@myref{#1}-\@myref{#2}}

\newcounter{list}[section]

\makeatother

\makeatletter

\newcommand{\pre}{\function{pre}}

\newcommand{\adde}{\function{add}}
\newcommand{\dele}{\function{del}}
\newcommand{\cost}{\function{cost}}

\def\hash{\text{\relsize{-1}\#}}
\newcommand{\ar}[1]{\hash{}#1}

\newcommand{\lsota}{state-of-the-art\xspace}  

\newcommand{\astar}{\ifmmode{A^*}\else{A$^*$}\fi\xspace}
\newcommand{\gbfs}{\ifmmode{\mathrm{GBFS}}\else{GBFS}\fi\xspace}
\NewDocumentCommand{\uct}{s}{\ifmmode{\mathrm{UCT}{\IfBooleanT{#1}{^*}}}\else{UCT{\IfBooleanT{#1}{*}}}\fi\xspace}
\NewDocumentCommand{\guct}{s}{\ifmmode{\mathrm{GUCT}{\IfBooleanT{#1}{^*}}}\else{GUCT{\IfBooleanT{#1}{*}}}\fi\xspace}

\newcommand{\topen}{tree-based open list\xspace}

\newcommand{\newheuristic}[2]{
 \def#1{
  \relax\ifmmode
  h^\mathrm{#2}\xspace
  \else
  \text{#2}\xspace
  \fi
 }
}

\newheuristic{\lmcut}{LMcut}
\newheuristic{\mands}{M\&S}
\newheuristic{\pdb}{PDB}
\newheuristic{\ff}{FF}
\newheuristic{\ce}{CEA}
\newheuristic{\cg}{CG}
\newheuristic{\gc}{GC}
\newheuristic{\ad}{add}
\newheuristic{\hmax}{max}
\newheuristic{\lc}{LC}
\newheuristic{\blind}{blind}

\newcommand{\newlearnedheuristic}[2]{
 \def#1{
  \relax\ifmmode
  H^\mathrm{#2}\xspace
  \else
  \text{#2}\xspace
  \fi
 }
}

\newlearnedheuristic{\Hlmcut}{LMcut}
\newlearnedheuristic{\Hmands}{M\&S}
\newlearnedheuristic{\Hpdb}{PDB}
\newlearnedheuristic{\Hff}{FF}
\newlearnedheuristic{\Hce}{CEA}
\newlearnedheuristic{\Hcg}{CG}
\newlearnedheuristic{\Had}{add}
\newlearnedheuristic{\Hmax}{max}
\newlearnedheuristic{\Hlc}{LC}
\newlearnedheuristic{\Hblind}{blind}

\newcommand{\newUnitCostHeuristic}[2]{
 \def#1{
  \relax\ifmmode
  \hat{h}^\mathrm{#2}\xspace
  \else
  \text{#2}\xspace
  \fi
 }
}

\newUnitCostHeuristic{\lmcuto}{LMcut}
\newUnitCostHeuristic{\mandso}{M\&S}
\newUnitCostHeuristic{\ffo}{FF}
\newUnitCostHeuristic{\ceo}{CEA}
\newUnitCostHeuristic{\cgo}{CG}
\newUnitCostHeuristic{\ado}{add}
\newUnitCostHeuristic{\gco}{GoalCount}
\newUnitCostHeuristic{\lco}{LC}

\newcommand{\newrandomheuristic}[2]{
 \def#1{
  \ifmmode
  \rh^\mathrm{#2}\xspace
  \else
  \text{#2}\xspace
  \fi
 }
}

\newrandomheuristic{\rlmcut}{LMcut}
\newrandomheuristic{\rmands}{M\&S}
\newrandomheuristic{\rpdb}{PDB}
\newrandomheuristic{\rff}{FF}
\newrandomheuristic{\rce}{CEA}
\newrandomheuristic{\rcg}{CG}
\newrandomheuristic{\rad}{add}
\newrandomheuristic{\rhmax}{max}
\newrandomheuristic{\rlc}{LC}

\makeatother

\usepackage{etoolbox}
\usepackage{xkeyval}
\makeatletter

\def\strips@initialize{
\def\@transitiononly{0}
\def\@conditiontype{0}
\def\@usecondeffect{0}
\def\@cost{0}
\def\@useaxiom{0}
\def\@lifted{0}
\def\@track{0}
}

\define@key{strips}{transition-only}[]{\def\@transitiononly{1}}
\define@key{strips}{negative-preconditions}[]{\def\@conditiontype{1}}
\define@key{strips}{disjunctive-preconditions}[]{\def\@conditiontype{2}}
\define@key{strips}{conditional-effects}[]{\def\@usecondeffect{1}}
\define@key{strips}{action-costs}[]{\ifnumcomp{\@cost}{<}{1}{\def\@cost{1}}{}}
\define@key{strips}{unitcost}[]{\ifnumcomp{\@cost}{<}{2}{\def\@cost{2}}{}}
\define@key{strips}{derived-predicates}[]{\def\@useaxiom{1}}
\define@key{strips}{lifted}[]{\def\@lifted{1}}
\define@key{strips}{optimal}[]{\ifnumcomp{\@track}{<}{1}{\def\@track{1}}{}}
\define@key{strips}{satisficing}[]{\ifnumcomp{\@track}{<}{2}{\def\@track{2}}{}}
\define@key{strips}{agile}[]{\ifnumcomp{\@track}{<}{3}{\def\@track{3}}{}}

\def\conditionset{
\if\@useaxiom0
P
\else
P\cup P_X
\fi
}

\let\satisfies@orig\satisfies

\def\satisfies{
\if\@conditiontype0
\supseteq
\else
\satisfies@orig
\fi
}

\def\condition{
\if\@conditiontype0
\conditionset
\else
\mathcal{F}(\conditionset)
\fi
}

\def\ga{
\if\@lifted0
a
\else
a^{\dagger}
\fi
}

\def\applyformula{
\if\@usecondeffect0
(s \setminus \dele(a)) \cup \adde(a)
\else
(s
 \setminus \braces{e \mid (c \triangleright e) \in \dele(\ga), c\satisfies s})
 \cup      \braces{e \mid (c \triangleright e) \in \adde(\ga), c\satisfies s}
\fi
}

\NewDocumentCommand{\strips}{O{}}{
\strips@initialize
\setkeys{strips}{#1}
\if\@lifted1
  \strips@propositional\par
  \strips@lifted
\else
  \strips@propositional
\fi
}

\newcommand{\strips@propositional}{
\if\@conditiontype1
Given a set of propositional variables $V$,
let $\mathcal{F}(V)$ be a propositional formula consisting of $V$ and
logical operations $\braces{\land,\lnot}$.
\fi
\if\@conditiontype2
Given a set of propositional variables $V$,
let $\mathcal{F}(V)$ be a propositional formula consisting of $V$ and
logical operations $\braces{\land,\lor,\lnot}$.
\fi
\if\@useaxiom0
We define a propositional STRIPS Planning problem
as a 4-tuple $\brackets{P,A,I,G}$
where
 $P$ is a set of propositional variables,
 $A$ is a set of actions,
 $I\subseteq P$ is the initial state, and
 $G\subseteq \conditionset$ is a goal condition.
\else
We define a propositional STRIPS Planning problem
as a 6-tuple $\brackets{P,A,X,P_X,I,G}$
where
 $P$ is a set of propositions,
 $A$ is a set of actions,
 $X$ is a set of axioms,
 $P_X$ is a set of derived propositions ($P\cap P_X=\emptyset$),
 $I\subseteq P$ is the initial state, and
 $G\subseteq \conditionset$ is a goal condition.
\fi
\ifnumcomp{\@transitiononly}{>}{0}{
We omit the details of action applications as they are not important in this paper.
It suffices to say an action $a\in A$ transitions from a state $s\subseteq P$ to a successor $s'=a(s)\subseteq P$.
}{
\ifnumcomp{\@cost}{<}{1}{
Each action $a\in A$ is a 3-tuple $\brackets{\pre(a),\adde(a),\dele(a)}$ where
}{
Each action $a\in A$ is a 4-tuple $\brackets{\pre(a),\adde(a),\dele(a),\cost(a)}$ where
$\cost(a) \in \Z^{0+}$ is a cost\ifnumcomp{\@cost}{=}{2}{ (We assume unit-cost: $\forall a\in A; \cost(a)=1$)}{},
}
$\pre(a) \subseteq \condition$ is a precondition and
\if\@usecondeffect0
$\adde(a), \dele(a)\subseteq P$ are the add-effects and delete-effects.
\else
$\adde(a), \dele(a)$ are the add-effects and delete-effects.
Each effect is denoted as $c \triangleright e$ where
$c \in \condition$ is an \emph{effect condition} and
$e \in P$.
\fi
\if\@useaxiom1
The set of axioms $X$ consists of clauses $f \Rightarrow p$ where
$f \in \condition$ is a body and $p \in P_X$ is a head.
\fi
A state $s\subseteq \conditionset$ is a set of true propositions
(all of $P\setminus s$ is false),
an action $a$ is \emph{applicable} when $s \satisfies \pre(a)$ (read: $s$ \emph{satisfies} $\pre(a)$),
and applying action $a$ to $s$ yields a new successor state
\if\@useaxiom0
$a(s) = \applyformula$.
\else
$a(s)$.
To compute $a(s)$, we first obtain a non-derived state
$s' \gets \applyformula \setminus P_X $.
Then we perform a fix-point calculation
$s' \gets s' \cup \braces{p \in P_X \mid (f\Rightarrow p)\in X \land s \satisfies f}$.
\fi
\par
}                               
The task of classical planning is to find a sequence of actions called a \emph{plan} $(\ga_1,\cdots,\ga_n)$
where, for $1\leq t\leq n$,
 $s_0=I$,
 \ifnumcomp{\@transitiononly}{>}{0}{}{$s_t\satisfies \pre(a_{t+1})$,}
 $s_{t+1}=a_{t+1}(s_t)$,
 and $s_n\satisfies G$.
\ifnumcomp{\@track}{>}{0}{
 A plan is \emph{optimal} if
 \ifnumcomp{\@cost}{<}{1}{
   there is no shorter plan.
 }{
   there is no plan with a lower cost $\sum_t \cost(a_t)$.
 }
 \ifnumcomp{\@track}{>}{1}{
   A plan is otherwise called \emph{satisficing}.
   \ifnumcomp{\@track}{>}{2}{
     A problem setting that completely ignores the solution quality is called an \emph{agile} setting,
     while a \emph{satisficing} setting implies that the solver still attempts to find a
     \ifnumcomp{\@cost}{<}{1}{shorter}{cheaper}
     plan.
     This paper focuses on the \emph{agile} setting.
   }{
     This paper focuses on the \emph{satisficing} setting.
   }
 }{}
}{}
}

\newcommand{\strips@lifted}{
In \emph{Lifted STRIPS}, each propositional variable is an \emph{instantiation}/\emph{grounding} of
a first-order logic predicate.
Each predicate $p(x_1,\ldots,x_{\ar{p}})$ is parameterized by a list of parameters/variables/arguments $X=(x_1,\ldots,x_{\ar{p}})$,
where $\ar{p}$ is an \emph{arity} of $p$.
A proposition is obtained by substituting each $x_i$ with an \emph{object} in a set $O$.
Each $p$ therefore has $O^{\ar{p}}$ instantiations.
Similarly, each action $a\in A$ is now called a \emph{ground action},
which is an instantiation of a \emph{lifted action} $a(x_1,\ldots,x_{\ar{p}})$ parameterized by $\ar{a}$ parameters.
A ground action is obtained by substituting the arguments as well as
the parameters used in the preconditions and the effects.
}

\makeatother

\usepackage{etoolbox}
\usepackage{xkeyval}
\makeatletter

\long\def\addto#1#2#3{
  \ifinlist{#3}{#1}{
  }{
    \listadd{#1}{#3}
    \ifdefempty#2{
     \expandafter\def\expandafter#2\expandafter{#2#3}
    }{
     \expandafter\def\expandafter#2\expandafter{#2,#3}
    }
  }
}

\define@key{heuristics}{ff}[1]{
 \def\heuristics@ff{#1}
 \ifnumcomp{\heuristics@ff}{>}{0}{
  \addto{\heuristiclist}{\heuristicstr}{\ff}
  \addto{\heuristiccitelist}{\heuristiccitestr}{hoffmann01}
 }{}
}
\define@key{heuristics}{ad}[1]{
 \def\heuristics@ad{#1}
 \ifnumcomp{\heuristics@ad}{>}{0}{
  \addto{\heuristiclist}{\heuristicstr}{\ad}
  \addto{\heuristiccitelist}{\heuristiccitestr}{bonet2001planning}
 }{}
}
\define@key{heuristics}{hmax}[1]{
 \def\heuristics@hmax{#1}
 \ifnumcomp{\heuristics@hmax}{>}{0}{
  \addto{\heuristiclist}{\heuristicstr}{\hmax}
  \addto{\heuristiccitelist}{\heuristiccitestr}{bonet2001planning}
 }{}
}
\define@key{heuristics}{gc}[1]{
 \def\heuristics@gc{#1}
 \ifnumcomp{\heuristics@gc}{>}{0}{
  \addto{\heuristiclist}{\heuristicstr}{\gc}
  \addto{\heuristiccitelist}{\heuristiccitestr}{FikesHN72}
 }{}
}
\define@key{heuristics}{cea}[1]{
 \def\heuristics@cea{#1}
 \ifnumcomp{\heuristics@cea}{>}{0}{
  \addto{\heuristiclist}{\heuristicstr}{\ce}
  \addto{\heuristiccitelist}{\heuristiccitestr}{helmert2008unifying}
 }{}
}
\define@key{heuristics}{cg}[1]{
 \def\heuristics@cg{#1}
 \ifnumcomp{\heuristics@cg}{>}{0}{
  \addto{\heuristiclist}{\heuristicstr}{\cg}
  \addto{\heuristiccitelist}{\heuristiccitestr}{Helmert04}
 }{}
}

\define@key{heuristics}{simplified}[1]{\def\heuristics@simplified{#1}}

\define@key{heuristics}{relaxation}[1]{\def\heuristics@relaxation{#1}}

\define@key{heuristics}{perfect}[1]{\ifnumcomp{\heuristics@properties}{<}{#1}{\def\heuristics@properties{#1}}{}}
\define@key{heuristics}{admissible}[2]{\ifnumcomp{\heuristics@properties}{<}{#1}{\def\heuristics@properties{#1}}{}}
\define@key{heuristics}{inadmissible}[3]{\ifnumcomp{\heuristics@properties}{<}{#1}{\def\heuristics@properties{#1}}{}}
\define@key{heuristics}{perfectsat}[4]{\ifnumcomp{\heuristics@properties}{<}{#1}{\def\heuristics@properties{#1}}{}}
\define@key{heuristics}{tdom}[5]{\ifnumcomp{\heuristics@properties}{<}{#1}{\def\heuristics@properties{#1}}{}}

\define@key{heuristics}{expansion}[1]{\def\heuristics@expansion{#1}}

\NewDocumentCommand{\heuristics}{O{}}{
\def\heuristics@ff{0}
\def\heuristics@ad{0}
\def\heuristics@hmax{0}
\def\heuristics@gc{0}
\def\heuristics@cea{0}
\def\heuristics@cg{0}
\def\heuristics@relaxation{0}
\def\heuristics@simplified{0}
\def\heuristics@properties{0}
\def\heuristics@expansion{0}
\def\heuristiclist{}
\def\heuristiccitelist{}
\def\heuristicstr{}
\def\heuristiccitestr{}
\setkeys{heuristics}{#1}
\ifnumcomp{\heuristics@simplified}{>}{0}{
A domain-independent heuristic function $h(s)$
returns an estimate of the cumulative cost from a state $s$ to one of the goal states (states that satisfy $G$).
}{
Given a problem $\brackets{P,A,I,G}$ and a state $s$,
a domain-independent heuristic function $h(s, \brackets{P,A,I,G})$
returns an estimate of the cumulative cost from $s$ to one of the goal states (states that satisfy $G$),
typically through a symbolic, non-statistical means including problem relaxation and abstraction.
It is often abbreviated as $h(s)$ or $h(s,G)$.
}
\ifdefempty\heuristiclist{}{
Notable \lsota functions that appear in this paper includes
$\heuristicstr$ \citep{\heuristiccitestr}.
}
\ifnumcomp{\heuristics@properties}{<}{1}{}{
  Often, the true optimal cost from a state $s$ is called
  the \emph{perfect heuristics} $h^*(s)$ \citep{helmert2008good}.
  \ifnumcomp{\heuristics@properties}{<}{2}{}{
    \emph{Admissible} heuristics are those which never overestimate $h^*$,
    i.e., $\forall s; h(s)\leq h^*(s)$.
    Optimizing algorithms like \astar \citep{hart1968formal} are
    guaranteed to find the optimal solutions with such heuristics.
    \ifnumcomp{\heuristics@expansion}{<}{1}{}{
      Moreover, \astar is the optimal expansion algorithm, i,e.,
      expands the fewest nodes among all algorithms under the same admissible $h$.
    }
    \ifnumcomp{\heuristics@properties}{<}{3}{}{
      Otherwise they are called \emph{inadmissible} heuristics,
      and are typically combined with satisficing/agile algorithms like GBFS \citep{doran1966experiments,bonet2001planning}.
      \ifnumcomp{\heuristics@properties}{<}{4}{}{
        Furthermore, heuristics that preserve the same ordering as $h^*$ are called
        \emph{perfect satisficing heuristics} $h^\leq$ \citep{kuroiwa2022biased},
        i.e., $\forall s,t; h^\leq(s)\leq h^\leq(t)\then h^*(s)\leq h^*(t)$.
        \ifnumcomp{\heuristics@expansion}{<}{1}{}{
          GBFS is the optimal expansion algorithm under $h^\leq$.
        }
        \ifnumcomp{\heuristics@properties}{<}{5}{}{
          Given a monotonic \emph{inflation} function $t:\R^{0+}\to\R^{0+}$
          s.t. $\forall x; t(x)\geq x$ and $\forall x,y; x\geq y \then t(x)\geq t(y)$,
          heuristics that preserve the same ordering as $h^*$ when inflated are called
          \emph{$t$-dominating heuristics},
          i.e., $\forall s,t; t(h(s))\leq h(t)\then h^*(s)\leq h^*(t)$.
        }
      }
    }
  }
}

\if\heuristics@relaxation1
A significant class of heuristics is called delete relaxation heuristics,
which solve a relaxed problem which does not contain delete effects,
and then returns the cost of the solution of the relaxed problem as an output.
The cost of the optimal solution of a delete relaxed planning problem from a state $s$ is
denoted by $h^+(s)$, but this is too expensive to compute in practice (NP-complete) \citep{bylander1996}.
Therefore, practical heuristics typically try to obtain its further relaxations
that can be computed in polynomial time.
\fi

\ifnumcomp{\heuristics@hmax}{>}{1}{
Max heuristics $\hmax$ \citep{bonet2001planning} is recursively defined as follows:
\begin{align}
 \hmax(s,G) = \max_{p\in G}
 \left\{
  \begin{array}{l}
   0\ \text{if}\ p\in s.\ \text{Otherwise,}\\
   \min_{\braces{a\in A\mid p\in\adde(a)}} \\
    \quad \left[\cost(a)+\ad(s, \pre(a))\right].
  \end{array}
 \right.
\end{align}
}{}

\ifnumcomp{\heuristics@ad}{>}{1}{
Additive heuristics $\ad$ \citep{bonet2001planning} is recursively defined as follows:
\begin{align}
 \ad(s,G) = \sum_{p\in G}
 \left\{
  \begin{array}{l}
   0\ \text{if}\ p\in s.\ \text{Otherwise,}\\
   \min_{\braces{a\in A\mid p\in\adde(a)}} \\
    \quad \left[\cost(a)+\ad(s, \pre(a))\right].
  \end{array}
 \right.
\end{align}
}{}

\ifnumcomp{\heuristics@ff}{>}{1}{
FF heuristics $\ff$ \citep{hoffmann01} is defined based on another heuristics $h$, such as $h=\ad$, as a subprocedure.
For each proposition $p$,
the action $a$ that adds $p$ with the minimal $\cost(a)+h(s, \pre(a))$
is conceptually ``the cheapest action that achieves a subgoal $p$'',
called the \emph{cheapest achiever} / \emph{best supporter} $\text{bs}(p,s,h)$ of $p$.
Using this, $\ff$ is defined as the sum of actions in a relaxed plan $\Pi^+$ constructed as follows:
\begin{align}
 \ff(s,G,h) &= \sum_{a\in \Pi^+(s,G,h)} \cost(a)\\
 \Pi^+(s,G,h) &= \bigcup_{p\in G}
 \left\{
  \begin{array}{l}
   \emptyset\ \text{if}\ p\in s.\ \text{Otherwise,}\\
   \braces{a} \cup \Pi^+(s,\pre(a)) \\
   \qquad \text{where}\ a=\text{bs}(p,s,h).
  \end{array}
 \right.\\
 \text{bs}(p,s,h)&=\argmin_{\braces{a\in A\mid p\in \adde(a)}} \left[\cost(a)+h(s, \pre(a))\right].
\end{align}
\ifnumcomp{\heuristics@ff}{>}{2}{
  In practice, $\ff$ can be implemented in several ways, each producing different values
  due to the tie-breaking difference in the $\argmin$ in $\text{bs}(p,s,h)$.
  \citet{kuroiwa2019case} showed that Graphplan-based implementation yields the best planner performance
  due to the combination of low-level efficiency and heuristic accuracy.
}{}
}{}

\ifnumcomp{\heuristics@gc}{>}{1}{
Goal Count heuristics $\gc$ is a simple heuristic proposed in \citep{FikesHN72}
that counts the number of propositions that are not satisfied yet.
$\brackets{\text{condition}}$ is a cronecker's delta / indicator function that returns 1 when the condition is satisfied.
\begin{align}
 \gc(s,G) &= \sum_{p\in G} \dbrackets{p\not \in s}.
\end{align}
}{}
}

\makeatother

\iftoggle{dev}{
  \IfFileExists{\jobname.needpyg}{
    \usepackage[finalizecache,cachedir=.]{minted}
  }{
    \usepackage[cachedir=.]{minted}
  }
}{
  \usepackage[frozencache,cachedir=.]{minted}
}

\allowdisplaybreaks

\title{Extreme Value Monte Carlo Tree Search for Classical Planning}

\author{
Masataro Asai\equalcontrib\textsuperscript{\rm 1},
Stephen Wissow\equalcontrib\textsuperscript{\rm 2}
}
\affiliations {
\textsuperscript{\rm 1}MIT-IBM Watson AI Lab\\
\textsuperscript{\rm 2}University of New Hampshire\\
masataro.asai@ibm.com,
sjw@cs.unh.edu
}

\begin{document}

\maketitle

\begin{abstract}
Despite being successful in board games and reinforcement learning (RL),
Monte Carlo Tree Search (MCTS) combined with Multi-Armed Bandits (MABs)
has seen limited success in domain-independent classical planning until recently.
Previous work \citep{wissow2024scale} showed that UCB1, designed for bounded rewards, does not perform well
as applied to cost-to-go estimates in classical planning, which are unbounded in $\R$,
and showed improved performance using a Gaussian reward MAB instead.
This paper further sharpens our understanding of ideal bandits for planning tasks.
Existing work has two issues:
first, Gaussian MABs under-specify the support of cost-to-go estimates as $(-\infty,\infty)$,
which we can narrow down.
Second, Full Bellman backup \citep{schulte2014balancing},
which backpropagates sample max/min, lacks theoretical justification.
We use \emph{Peaks-Over-Threashold Extreme Value Theory} to resolve both issues at once,
and propose a new bandit algorithm (UCB1-Uniform).
We formally prove its regret bound and
empirically demonstrate its performance in classical planning.
\end{abstract}

\section{Introduction}
\label{sec:intro}

A recent breakthrough \citep{wissow2024scale} in Monte Carlo Tree Search (MCTS) combined with Multi-Armed Bandit (MAB)
demonstrated that a better theoretical understanding of bandit-based algorithms
can significantly improve search performance in classical planning \citep{FikesHN72}.
Building upon the Trial-Based Heuristic Tree Search (THTS) framework \citep{schulte2014balancing}, \citet{wissow2024scale} showed why the UCB1 bandit \citep{auer2002finite} does not perform well in classical planning:
UCB1 assumes a reward distribution with a known, fixed, finite support (a mathematical term for a defined range
such as $[0,1]$) that is shared by all arms,
incorrectly assuming that cost-to-go estimates (heuristic values) always fall in this particular range.
They then proposed UCB1-Normal2 bandit that assumes a Gaussian reward distribution
which has an infinite support $(-\infty,\infty)$ that is impossible to violate,
and has a regret bound that can become constant when applied to deterministic state space search, as in classical planning.

We build on these advances to further our understanding of the strengths and requirements of MABs as applied to heuristic search, and in particular to
resolve two theoretical issues in previous work in this area.
The first is UCB1-Normal2's assumption that
cost-to-go estimates fall anywhere in $(-\infty,\infty)$, which is an under-specification that can be narrowed down to $[0,\infty)$ or even further.
The second is the insufficient statistical characterization of \emph{extrema} (maximum/minimum)
in so-called \emph{Full Bellman} backup \citep{schulte2014balancing}
that backpropagates the smallest/largest mean among the arms.
\citeauthor{schulte2014balancing} informally criticized
the use of averages in UCT as ``rather odd'' for planning,
but without bandit-theoretic justifications.

This paper introduces Extreme Value Theory \citep[EVT]{beirlant2004statistics,dehaan2006extreme}
as the statistical foundation for understanding general optimization tasks.
EVTs are designed to model the statistics of extrema of distributions
using the \emph{Extremal Limit Theorems},
unlike most statistical literature that models the \emph{average} behavior
based on the \emph{Central Limit Theorem} \citep[CLT]{laplace1812centrallimittheorem}.
Among branches of EVTs, we identified \emph{Peaks-Over-Threashold EVT} \citep{pickands1975statistical,balkema1974residual}
as our primary tool for designing new algorithms,
leading us to the Generalized Pareto (GP) distribution,
which plays the same role in EVT as the Gaussian distribution does in the CLT.
Based on this framework,
we propose a novel MAB algorithm called UCB1-Uniform
for heuristic search applied to classical planning,
using the fact that the Uniform distribution is a special case of the GP distribution
to avoid the numerical difficulty of estimating the latter's parameters.
We propose a novel heuristic search algorithm for classical planning, GreedyUCT-Uniform (\guct-Uniform), an MCTS that leverages UCB1-Uniform.

We compared \guct-Uniform's performance against
various existing bandit-based MCTS algorithms,
traditional Greedy Best First Search \citep[\gbfs]{bonet2001planning,doran1966experiments}, and a
\lsota diversified search algorithm called Softmin-Type(h) \citep{kuroiwa2022biased}.
The results showed that our algorithm outperforms existing \lsota algorithms across diverse heuristics.
For example,
under the same evaluation budget of $10^4$ nodes with the $\ff$ heuristic \citep{bonet2001planning},
\guct-Uniform solved
67.8, 23.4, and 33.2 more instances than
GBFS, \guct-Normal2, and Softmin-Type(h), respectively.
\guct-Uniform also significantly outperformed MCTS variants combined with Max-$k$ bandits \citep{cicirello2004heuristic},
a bandit paradigm whose objective differs significantly from those of classical planning.
MCTS combined with Max-$k$ bandits
(MaxSearch \citep{streeter2006simple}, RobustUCT \citep{bubeck2013bandits}, and Threshold Ascent \citep{kikkawa2022materials})
performed poorly in the classical planning task,
outperformed by \guct-Uniform by more than 300 instances.
Our code is published at \url{github.com/IBM/pyperplan-mcts-public}.
A full version of the paper with appendix is on arxiv:2405.18248.

\section{Preliminaries}
\label{sec:background}

\label{sec:classical-planning}

\strips[transition-only,agile]

A domain-independent heuristic function $h$ in classical planning is
a function of a state $s$ and the problem $\brackets{P,A,I,G}$,
though the notation $h(s)$ usually omits the latter, that
returns an estimate of the cumulative cost of a sequence of actions transitioning from $s$ to a goal state $s_g \satisfies G$.
Details of specific heuristic functions are beyond the scope of this paper, and are included in the appendix.

\subsection{Multi-Armed Bandit (MAB)}

MAB \citep{thompson1933likelihood,robbins1952some,bush1953stochastic} is the problem of
finding the best strategy to choose from multiple unknown reward distributions.
It is typically depicted by a row of $K$ slot machines each with a lever or `arm.'
Each time the player pulls an arm (a \emph{trial}),
they receive a reward sampled from that arm's reward distribution.
Through multiple trials, the player discovers the arms' distributions and selects arms to maximize the reward.

The most common optimization objective of MAB is \emph{Cumulative Regret} (CR) minimization.
Let $\rr_{it}$ ($1\leq i \leq K$) be a random variable (RV) for the reward received from the $t$-th pull of an arm $i$.
$\rr_{it}$ follows an unknown \emph{reward distribution} $p(\rr_{i})$ which stays the same over $t$.
Let $t_{i}$ be the number of pulls on arm $i$
when $T=\sum_i t_{i}$ pulls are performed in total.
\begin{defi}
 Let $I_t$ be the arm pulled at $t$.
 The \emph{cumulative regret} $\Delta$ is the gap between the optimal and the actual expected cumulative reward:
 $\textstyle\Delta=\max_i \E[\sum_{t=1}^T \rr_{it}] - \E[\sum_{t=1}^T \rr_{I_tt}]$.
 \label{def:cr}
\end{defi}
A regret bound indicates the \emph{speed} of convergence.
Algorithms with a logarithmically upper-bounded regret, $O(\log T)$,
are called \emph{asymptotically optimal} because
this is the theoretical optimum achievable by any algorithm \citep{lai1985asymptotically}.

Upper Confidence Bound 1 \citep[UCB1]{auer2002finite} is
a logarithmic CR MAB for rewards $\rr_i\in [0,c]$ with a known $c$.
Let $r_{i1},\ldots, r_{it_i}\sim p(\rr_i)$ be $t_i$ \iid samples obtained from an arm $i$.
Let $\hat{\mu}_i=\frac{1}{t_i}\sum_{j=1}^{t_i} r_{ij}$.
To minimize CR,
UCB1 selects $i$ with the largest Upper Confidence Bound value $\text{UCB1}_i$:
\begin{align}
 \begin{split}
  \text{UCB1}_i &\textstyle= {\hat{\mu}_i + c\sqrt{{(2\log T)}/{t_i}}}\\
  \text{LCB1}_i &\textstyle= {\hat{\mu}_i - c\sqrt{{(2\log T)}/{t_i}}}
 \end{split}
 \label{eq:ucb1}
\end{align}
For reward (cost) minimization,
we can select $i$ with the smallest $\text{LCB1}_i$ value defined above
(e.g., in \citet{kishimoto2022bandit}),
but we may use the terms U/LCB1 interchangeably.

U/LCB1's second term is often called an \emph{exploration term}.
In practice,
$c$ is often set heuristically as a hyperparameter and referred to as the \emph{exploration rate},
ignoring the original theoretical meaning as the upper limit of support $[0,c]$.

U/LCB1 refers to a specific algorithm proposed by \citet{auer2002finite}, 
while U/LCB refers to general upper/lower confidence bounds of random variables.
Often an LCB subtracts the exploration term instead of adding it as in a UCB.

\subsection{Forward Heuristic Best-First Search}
\label{sec:mcts}

Classical planning problems are typically solved as a path finding problem
defined over a state space graph induced by the transition rules,
and the current dominant approach is based on \emph{forward search}.
Forward search maintains a set of search nodes called an \emph{open list},
and repeatedly
(1) (\emph{selection}) selects a node from the open list,
(2) (\emph{expansion}) generates its successor nodes,
(3) (\emph{evaluation}) evaluates the successor nodes, and
(4) (\emph{queueing}) reinserts them into the open list.
Termination typically occurs when the node selected for expansion satisfies a goal condition,
but a satisficing/agile algorithm can perform \emph{early goal detection},
which immediately checks whether any successor node generated in step (2) satisfies the goal condition.
Since this paper focuses on agile search, we use early goal detection for all algorithms.

Within forward search,
forward \emph{best-first} search defines a particular ordering in the open list
by defining \emph{node evaluation criteria} (NEC) $f$ for selecting the best node in each iteration.
Let us denote a node by $n$ and the state represented by $n$ as $s_n$.
As NEC,
Dijkstra search \citep{dijkstra1959note} uses $f_{\mathrm{Dijkstra}}(n)=g(n)$ ($g$-value), the minimum cost from the initial state $I$ to the state $s_n$ found so far.
\astar \citep{hart1968formal} uses $f_{\astar}(n)=g(n)+h(s_n)$, the sum of $g$-value and the value returned by a heuristic function $h$ ($h$-value).
Greedy Best First Search \citep[\gbfs]{bonet2001planning} uses $f_{\gbfs}(n)=h(s_n)$.
Forward best-first search that uses $h$ is called forward \emph{heuristic} best-first search.
Dijkstra search is a special case of \astar with $h(s)=0$.

Typically, an open list is implemented as a priority queue ordered by NEC.
Since the NEC can be stateful, e.g., $g(s_n)$ can update its value,
a priority queue-based open list, depending on implementation, may have unfavorable time complexity for removals and thus may assume monotonic updates to the NEC. 
\astar, Dijkstra, and \gbfs satisfy this condition because
$g(n)$ decreases monotonically and $h(s_n)$ is constant.

MCTS is a class of forward heuristic best-first search
that represents the open list as the leaves of a tree.
We call such a tree a \emph{\topen}.
Our MCTS is based on the description in \citet{keller2013trial} and \citet{schulte2014balancing},
whose implementation details are available in the appendix.
Overall, MCTS works in the same manner as other best-first searches with a few key differences.
(1) (\emph{selection}) To select a node from the \topen,
it recursively selects an action at each depth level of the tree,
starting from the root, using the NEC
to select a successor node,
descending until reaching a leaf node.
(Sometimes the action selection rule is also called a \emph{tree policy}.)
At the leaf, it
(2) (\emph{expansion}) generates successor nodes,
(3) (\emph{evaluation}) evaluates the new successor nodes,
(4) (\emph{queueing}) attaches them to the leaf, and
\emph{backpropagates} (or \emph{backs-up}) the information to the leaf's ancestors, all the way up to the root.

The evaluation obtains a heuristic value $h(s_n)$ of a leaf node $n$.
In adversarial games like Backgammon or Go, it is obtained either by
(1) hand-crafted heuristics,
(2) \emph{playouts} (or \emph{rollouts})
where the behaviors of both players are simulated
by (e.g.~uniformly) random actions (\emph{default policy}) until the game terminates,
or (3) a hybrid \emph{truncated simulation},
which returns a hand-crafted heuristic after performing a short simulation \citep{gelly2011monte}.
In recent work, the default policy is replaced by a learned policy \citep{alphago}.

Trial-based Heuristic Tree Search \citep[THTS]{keller2013trial,schulte2014balancing},
an MCTS for classical planning,
is based on two key observations:
(1) the rollout is unlikely to terminate in classical planning due to sparse goals,
unlike adversarial games, like Go, which are guaranteed to finish in a known number of steps with a clear outcome (win/loss); and
(2) a \topen can efficiently reorder entire subtrees of nodes, and
thus is more flexible than a priority queue-based open list, and
can readily implement traditional algorithms such as \astar and \gbfs without significant performance penalty.
In this paper, we use THTS and MCTS interchangeably.

Finally, Upper Confidence Bound applied to trees \citep[UCT]{kocsis2006bandit}
is an MCTS that uses the UCB1
Multi-Armed Bandit
algorithm for action selection and became widely popular in adversarial games.
\citet{schulte2014balancing} proposed several variants of UCT including GreedyUCT (\guct),
which differs from UCT in that
the NEC assigned to the node is simply its heuristic value $h(s_n)$ just like in GBFS,
rather than the $f$-value $(f = g(n)+h(s_n))$.
This paper
only discusses the greedy variants
due to our focus on agile planning.

\section{Heuristic Search with MABs}
\label{sec:issues}

We first revisit \gbfs and GUCT 
 from an MAB perspective.
While \citet{keller2013trial} generalized various algorithms
focusing on the procedural aspects (e.g., recursive backup),
we focus on their mathematical meaning.
\begin{defi}[NECs]
Let
$S(n)$ be the successors of a node $n$,
$L(n)$ be the leaf nodes in the subtree under $n$.
The NEC of GBFS and \guct are shown below,
where
$p$ is $n$'s parent, and thus $|L(p)|$ and $|L(n)|$ correspond to $T$ and $t_i$ in \refeq{eq:ucb1}.
\begin{align*}
 f_{\gbfs}(n)
 &=h_{\gbfs}(n) \\
 f_{\guct}(n)
 &=h_{\guct}(n) -c\sqrt{{(2\log |L(p)|)}/{|L(n)|}}
\end{align*}
\end{defi}
MCTS/THTS computes 
$h_{\gbfs}(n)$, $h_{\guct}(n)$ using backpropagation.
Below,
we expand the definitions of two backup functions presented by \citet{keller2013trial}
recursively down to the leaves,
assuming $h_{\gbfs}(n)=h_{\guct}(n)=h(s_{n})$ if $n$ is a leaf.
\begin{defi}[Full Bellman Backup]
\begin{align*}
 h_{\gbfs}(n)
 &\textstyle= \min_{n'\in S(n)} [h_{\gbfs}(n')]\\
 &\textstyle= \min_{n'\in S(n)} [\min_{n''\in S(n')} [h_{\gbfs}(n'')]]\\
 =\ldots
 &\textstyle= \min_{n'\in L(n)} [h(s_{n'})].
\end{align*}
\end{defi}
\begin{defi}[Monte Carlo Backup]
\begin{align*}
 h_{\guct}(n)
 =& \textstyle \sum_{n'\in S(n)} \frac{|L(n')|}{|L(n)|} h_{\guct}(n') & \\
 =  \textstyle \sum_{n'\in S(n)} &\textstyle \frac{\cancel{|L(n')|}}{|L(n)|} \sum_{n''\in S(n')} \frac{|L(n'')|}{\cancel{|L(n')|}} h_{\guct}(n'') & \\
 =\ldots
 =& \textstyle \frac{1}{|L(n)|}\sum_{n'\in L(n)} h(s_{n'}). &
\end{align*}
\end{defi}

\noindent Notice that each backup
is equivalent to simply computing the minimum or the weighted mean over all leaves in the subtree,
where each leaf $n'$ has $|L(n')|=1$ in classical planning.
In other words,
the set $\braces{h(s_{n'}) \mid n'\in L(n) }$ 
 is a \emph{reward dataset},
the heuristic $h(s_{n'})$ at each leaf $n'$ is a \emph{reward sample} in the dataset,
and the NECs use their \emph{statistics},
such as the mean and the minimum, estimated by Maximum Likelihood Estimation (MLE).
Backpropagation is just an effective way to update and cache the statistics.

\begin{theo}
 \label{theo:mle-gaussian}
 Given \iid $x_1,\ldots,x_N\sim \N(\rx|\mu,\sigma)$ (i.e., $\rx\sim\N(\mu,\sigma)$),
 the MLEs of $\mu$ and $\sigma$ are
 the empirical mean $\hat{\mu}=\frac{1}{N}\sum_i x_i$ and
 variance $\hat{\sigma}^2=\frac{1}{N-1}\sum_i (x_i-\hat{\mu})$.
 (Well-known result. Educational proof in appendix.)
\end{theo}

Understanding each $h(s)$
as a sample of a random variable representing a reward for MABs
makes it clear that existing MCTS/THTS for classical planning
fails to leverage the theoretical efficiency guarantees
from the rich MAB literature.
For example,
if we apply UCB1 to heuristic values in classical planning,
UCB1 no longer guarantees asymptotically optimal convergence toward the best arm
because
it incorrectly assumes $h \in [0,c]$ for a fixed hyperparameter $c$,
i.e., that $h$ has an \emph{a priori} known constant range $[0,c]$,
which in fact does not exist ($h$ varies significantly across states, and can be $\infty$).
This cannot be fixed by simply making $c$ larger.

\citet{wissow2024scale} proposed \guct-Normal and \guct-Normal2,
MCTS algorithms for classical planning that use Gaussian bandits
UCB1-Normal (\refeq{eq:gaussian-bandits-auer}) \citep{auer2002finite} and
UCB1-Normal2 (\refeq{eq:gaussian-bandits}),
motivated by
the Gaussian distribution's inclusive support range $\R$,
and in effect these bandits use the $h$ sample variance to dynamically estimate UCB1's $c$, i.e., the `exploration rate'.
Given $\rr_i\sim \N(\mu_i,\sigma_i)$,
let $\hat{\mu}_i$ and $\hat{\sigma}_i$ be the MLEs of $\mu_i,\sigma_i$ of arm $i$.
\begin{align}
 \text{U/LCB1-Normal}_i  &= \textstyle \hat{\mu}_i \pm \hat{\sigma}_i \sqrt{{(16\log T)}/{t_i}}.\label{eq:gaussian-bandits-auer}\\
 \text{U/LCB1-Normal2}_i &= \textstyle \hat{\mu}_i \pm \hat{\sigma}_i \sqrt{2\log T}.
 \label{eq:gaussian-bandits}
\end{align}
Each $(\hat{\mu}_i, \hat{\sigma}_i)$ corresponds to the average and the standard deviation of
the dataset $\braces{h(s_{n'}) \mid n'\in L(n) }$ of a node $n$.
\guct-Normal2 outperformed GBFS, GUCT, GUCT-Normal, and other variance-aware bandits.

Although \guct-Normal2 explored better than existing algorithms
while not violating assumptions about the reward range,
it still does not fully characterize the nature of heuristic functions,
as we describe below.

\paragraph{Under-Specification}
\label{sec:issues-underspecifictation}

\guct over-specifies the rewards to be in a fixed range $[0,c]$.
While $\N(\mu,\sigma)$ does not have this issue,
its support $\R=(-\infty,\infty)$ is an under-specification because
heuristic values are \emph{non-negative}, $\R^+=[0,\infty)$.

Moreover, the range can be narrowed down further.
For example, the FF heuristic \citep{hoffmann01} satisfies $\ff \in [h^+,\infty)$,
i.e., lower bounded by optimal delete relaxation heuristic $h^+$,
though this value is \textbf{NP}-complete to compute \citep{bylander1994} and thus in practice \emph{unknown}.
Similarly, the $\hmax$ heuristic is bounded by $[0,h^+]$.
Finally,
$h^+$ can be $\infty$ when the state is at a dead-end.
This indicates that the support of a heuristic function is generally \emph{unknown and half-bounded}.
Choosing an appropriate distribution, and a corresponding MAB that correctly leverages its properties,
should make MCTS faster.
A similar statistical modeling flaw was recently discussed in supervised heuristic learning \citep{nunez2024using}.

\paragraph{Estimating the Minimum}
\label{sec:issues-minimum}
Another issue in existing work is the use of the minimum (Full Bellman backup) in the \guct*-family of algorithms \cite{schulte2014balancing},
which lacks statistical justification,
in particular a theoretical explanation of \emph{why} using the minimum is allowed.
Regardless of whether rewards have finite-support or follow a Gaussian distribution,
the mean $\mu_i$, not the minimum, is inextricable from the design of and regret bound proofs for UCB1/-Normal/2.
In contrast, the theoretical framework we present in the next section addresses this
conflict.

One candidate for addressing the extrema of reward distributions was
proposed by \citet{tesauro2010bayesian}.
Given two Gaussian RVs $X_1\sim \N(\mu_1,\sigma^2_1)$ and $X_2\sim \N(\mu_2,\sigma^2_2)$,
the backup uses their maximum $\max(X_1,X_2) \sim \mathcal{N}(\mu_3,\sigma_3)$ where
$\mu_3=\mu_1\Phi(\alpha)+\mu_2\Phi(-\alpha)+(\sigma_1^2+\sigma_2^2)\phi(\alpha)$ and
$\sigma_3=(\mu_1^2+\sigma_1^2)\Phi(\alpha)+(\mu_2^2+\sigma_2^2)\Phi(-\alpha)+(\mu_1+\mu_2)(\sigma_1^2+\sigma_2^2)\phi(\alpha)$,
where $\Phi$ and $\phi$ are the CDF and the PDF of $\N(0,1)$ \citep{clark1961greatest}.
Unfortunately, this is merely an approximation if we combine the estimates iteratively for more than two arms,
as noted in \shortcite{tesauro2010bayesian}.
In our experiments,
we implemented this backup and call it GUCT$^+$ variants.

\paragraph{Dead-End Removal}
\label{sec:issues-deadend}
$\ff$ lacks an upper bound and could return $\infty$ when a node is a dead-end,
which is problematic for MABs that use the average of rewards, including both UCB1 and UCB1-Normal2.
Imagine that an arm returned rewards 3, 5, 4 and $\infty$ in order.
Once an arm receives the fourth reward $\infty$,
then the estimated mean suddenly becomes $\infty$ regardless of those observed previously
---i.e., $\frac{3+5+4+\infty}{4}=\infty$---
falsely masking potential solutions in the subtree.
\citet{schulte2014balancing} recognized this issue and
decided to exclude $\infty$ from rewards by removing dead-end nodes from the tree,
but this approach lacks statistical, MAB-theoretic justification as to why such a removal of a value from the sample is allowed.

\section{Extreme Value Theory (EVT)}
\label{sec:extreme-mcts}

To address these theoretical issues,
we use Peaks-Over-Threashold Extreme Value Theory (POT EVT).
\label{sec:evt-background}
Regular statistics are typically built around the Central Limit Theorem \citep[CLT]{laplace1812centrallimittheorem},
which deals with the limit behavior of the average of samples.
In contrast, a branch of statistics 
called \emph{Extreme Value Theory} \citep[EVT]{beirlant2004statistics,dehaan2006extreme}
describes the limit behavior of the maximum of samples.
EVT has been historically used for safety-critical applications
whose worst case behaviors matter.
For example, in hydrology, the estimated annual maximum water level of a river is used to
decide the height for an embankment.
We explain the EVT by way of first reviewing the CLT.

\begin{defi}
 A series of functions $f_n$ \emph{converges pointwise} to $f$ ($f_n\to[n\to\infty] f$)
 when
 $\forall x; \forall \epsilon; \exists n; |f_n(x)- f(x)|<\epsilon$.
\end{defi}

\begin{defi}
 A series of RVs $\rx_n$ \emph{converges in distribution} to a RV $\rx$
 if $p\parens{\rx_n}=f_n \to[n\to\infty] f=p\parens{\rx}$,
 denoted as $\rx_n\to[D] \rx$.
\end{defi}

\begin{theo}[CLT]
 Let $\rx_1,\ldots \rx_n$ be i.i.d., $\forall i;\E[\rx_i]=\mu$, $\Var[\rx_i]=\sigma^2$.
 Then
 $\sqrt{n}\parens{\frac{\sum_i \rx_i}{n}-\mu} \to[D] \ry\sim\N(0,\sigma)$. \label{theo:clt}
 I.e.,
 \ul{if $\rx_i$'s distribution $p(\rx_i)$ has a finite mean and variance},
 the average of $\rx_1\ldots\rx_n$ converges ($n\to\infty$) in distribution to a Gaussian with the same mean/variance,
 \ul{regardless of other details of $p(\rx_i)$, e.g., shape or support.}
\end{theo}

A common misunderstanding is that CLT assumes each RV $\rx_i$ to follow a Gaussian (untrue).
\ul{CLT's strength comes from its minimal assumption that $\rx_i$ are i.i.d. and share a finite $\mu$ and $\sigma$,
and nothing else.}
$\rx_i$ can follow Laplace, but not Cauchy (mean and variance are undefined).
In heuristic search, $\rx_i$ is a random choice from $\braces{h(s_{n'})}$ in a subtree of a node $n$.
Its mean and variance must be finite;
therefore $\braces{h(s_{n'})}$ should not contain $\infty$ as it makes the average $\infty$.
But it does not require each $h$ to follow a Gaussian
(each $h$ is indeed a Dirac delta $\delta(\rx=h)$, a deterministic value of a state),
nor the histogram of $\braces{h(s_{n'})}$ to resemble a Gaussian; Only the mean and the variance matter.

EVT has two limit theorems similar to the CLT, called the Extremal Limit Theorems (ELT).
The first kind, the \emph{Block Maxima} ELT
\citep{fisher1928limiting,gnedenko1943limiting},
states that
the maximum of \iid RVs converges in distribution to an \emph{Extreme Value Distribution}.
Given multiple subsets of data points,
it models the maximum of the next subset (\emph{block maxima}),
e.g., it predicts the maximum of the next month from the maxima of past several months.
However, what we use is
the second kind,
\emph{Peaks-Over-Threshold} (POT) ELT
\citep{pickands1975statistical,balkema1974residual}, which states that
the excesses of \iid RVs over a sufficiently high threshold $\theta$
converge in distribution to a Generalized Pareto (GP) distribution (\refig{fig:gp}),
predicting \emph{future excesses over $\theta$}.

\begin{defi}[Generalized Pareto Distribution]
 \label{def:gp}
 \begin{align*}
  \gp(\rx\mid\theta,\sigma,\xi)
  &=
  \left\{
  \begin{array}{ll}
   \frac{1}{\sigma}\parens{1+\xi\frac{\rx-\theta}{\sigma}}^{-\frac{\xi+1}{\xi}}& (\xi\not=0)\\
   \frac{1}{\sigma}\exp \parens{-\frac{\rx-\theta}{\sigma}}          & (\xi=0)
  \end{array}
  \right. (\rx>\theta)
 \end{align*}
\end{defi}

\begin{figure}[tb]
 \includegraphics[width=\linewidth]{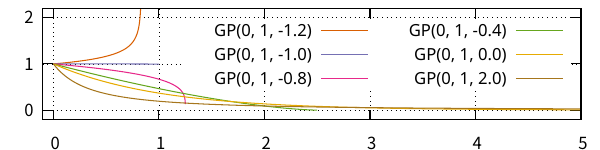}
 \caption{Generalized pareto distribution $\gp(0,1,\xi)$.}
 \label{fig:gp}
\end{figure}

\begin{figure}[tb]
 \includegraphics[width=\linewidth]{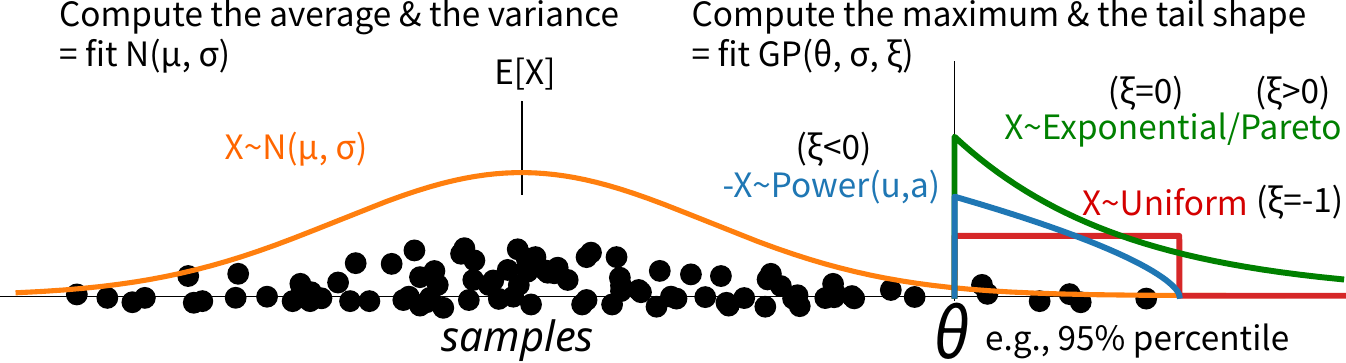}
 \caption{
 Computing the average and the variance is seen as fitting $\N(\mu,\sigma)$;
 Computing the maximum and the shape of the tail distribution is seen as fitting
 $\gp(\mu,\sigma,\xi)$ with $\xi<0$.
 }
 \label{fig:clt-vs-evt}
\end{figure}

\begin{theo}[Pickands--Balkema--de\,Haan\,theorem]
 \label{theo:evt-type2}
 Let $\rx_1,\ldots, \rx_n \sim p(\rx)$ be i.i.d. RVs
 and $\rx_{k,n}=\theta$ be their top-$k$ elements.
 As $n\to \infty$, $k\to \infty$, $\frac{k}{n}\to 0$ ($k\ll n$), then
 $p(\rx\mid\rx>\theta) \to[D] \gp(\rx\mid \theta,\sigma,\xi)$
 for some $\sigma\in\R^+,\xi \in \R$,
 \ul{regardless of other details of $p(\rx_i)$, e.g., shape or support.}
\end{theo}

$\theta$, $\sigma$, and $\xi$ are called the location, scale, and shape parameters.
GP has support $\rx\in [\theta,\theta-\frac{\sigma}{\xi}]$ when $\xi<0$, otherwise $\rx\in [\theta,\infty)$.
The shape dictates the tail behavior:
$\xi>0$ corresponds to a heavy-tailed distribution,
$\xi<0$ corresponds to a short-tailed distribution (i.e., has an upper limit),
and $\xi=0$ corresponds to an Exponential distribution.
Pareto, Exponential, Reverted Power, and Uniform distributions are special cases of $\gp$.

\refig{fig:clt-vs-evt} shows a conceptual illustration of POT.
In the standard statistical modeling,
practitioners often compute the average and the standard deviation of the data to fit $\N(\mu,\sigma)$,
which models the ``normative'' behavior of samples that appears at the center of the distribution.
In contrast, POT models rare events occuring in the \emph{tail distribution}.
Practitioners first extract a top-$k$ subset of samples in various ways, e.g.,
setting a threshold $\theta$, selecting the top 5\%, or directly specifying $k$,
then fit the parameters $\sigma,\xi$ of $\gp(\theta,\sigma,\xi)$ on this subset,
which predicts future excesses.
GP is accurate when
we retain (extract) enough data $k\to \infty$ as $n\to \infty$
while ignoring almost all data ($\frac{k}{n}\to 0$).
For example, estimates from top-1\% examples ($\frac{k}{n}=0.01$)
tend to be more accurate than those from top-5\% ($\frac{k}{n}=0.05$),
if $k$ is the same.

EVTs are appealing because \lsota search algorithms such as \gbfs are based on the minimum.
It is also worth noting that the short-tailed $\gp$ ($\xi<0$) resolves the shortcomings discussed in \refsec{sec:issues}.
Consider a maximization scenario (as GP models the maxima), where
we negate the heuristic values into rewards $-\ff\in (-\infty,-h^+]$.
By fitting $\sigma$ and $\xi$ to the data,
a short-tailed $\gp$ gives us an upper support $\theta-\frac{\sigma}{\xi}$,
which works as an estimate of $-h^+$.
$\gp$ also \emph{justifies} discarding $-\ff$ below $\theta$,
including the dead-ends $-\ff=-\infty$,
because $\gp$ is conditioned (only) by $\rx>\theta$.

\label{sec:estiamte-gpd}

One difficulty of $\gp$ is its parameter estimation,
which has been extensively studied with varying success
\citep{smith1987estimating,hill1975simple,resnick1997smoothing,hosking1987parameter,diebolt2005quasi,sharpe2021estimation}.
To avoid this,
we focus on the uniform distribution $U(l,u)$ with an unknown support $[l,u]$,
sacrificing one degree of freedom:
It is a special case with $\xi=-1$.
Note that \ul{POT does not assume any distribution (just like CLT);
i.e., we do not assume heuristic values follow $U(l,u)$ or GP.}

\begin{defi}
 The Uniform distribution is defined as follows.
\begin{align*}
 U(\rx|l,u)&\textstyle=\frac{1}{u-l}.\quad (l \le x \le u) \qquad \E[\rx]=\frac{l+u}{2}.
\end{align*}
\end{defi}

\begin{theo}
 $\gp(\theta,\sigma,-1)=U(\theta,\theta+\sigma)$. (proof omitted)
\end{theo}

One last bit of the detail is how to set $\theta$, but we actually do not use any explicit threshold.
Observe that in heuristic search, the search is already heavily focused toward the goal.
The search space ($n$ nodes) is exponentially large and
mostly unexplored, so the observed nodes ($k$ nodes) are the tiny fraction ($k \ll n$)
with small heuristic values relative to the entire state space.
Visiting a state $s$ with $h(s)>h(I)$ of initial state $I$ tends to be rare.
We confirmed this with GBFS+$\ff$ in the benchmark (\refsec{sec:experiments}):
Only 3.3\% of the evaluated nodes had $h(s)>h(I)$,
thus bad nodes are already rare in the reward set.
As a result,
our implementation omits explicit filtering of samples,
relying on implicit filtering from not expanding such bad states.
The only explicit rule is the dead-end removal.
Future work could theoretically justify the pruning with the cost of an
incumbent solution in iterated anytime search \cite{richter2010joy}.

\section{Bandit for Uniform Distributions}
\label{sec:evt-for-mcts}

To define our POT-based search algorithm,
we first review the Maximum Likelihood Estimates (MLEs) of Uniform distributions,
then propose a bandit that uses these estimates,
which is then used by MCTS as its NEC.

\begin{theo}[MLE of Uniform]
 \label{theo:mle-uniform}
 Given \iid $x_1,\ldots,x_N\sim U(\rx|l,u)$,
 the MLEs are
 $\hat{u}=\max_i x_i$ and $\hat{l}=\min_i x_i$.
 (Well-known result. Educational proof in appendix.)
\end{theo}

In MCTS,
we backpropagate these estimates from the leaves to the root: i.e.,
for $\hat{l}$ and $\hat{u}$
we use Full Bellman backup (use the minimum/maximum among the children).
This provides theoretical guidance on when Full Bellman backup is appropriate:
while Full Bellman backup is a method for efficiently estimating $U(u,l)$ of each node from its subtree,
\guct* uses Full Bellman Backup with an MAB designed for the wrong distribution (a distribution with a known fixed support $[0,c]$), and
therefore does not perform well.
To address this shortcoming, we propose a new MAB for $U(u,l)$:

\begin{theo}[Main results]
 In each trial $t$,
 assuming the reward $\rr_{i}$ of arm $i$ follows
 a Uniform distribution with an unknown support
 $U(l_i,u_i)$, 
 the \emph{U/LCB1-Uniform} policy respectively pulls the arm $i$ that maximizes/minimizes
 \begin{align*}
  \begin{array}{llclc}
   \text{\rm U/LCB1-Uniform}_i &= \frac{\hat{u}_i+\hat{l}_i}{2}          \pm (\hat{u}_i-\hat{l}_i) \sqrt{6t_i\log T} \\
  \end{array}
 \end{align*}
 where $\hat{l}_i=\min_j r_{ij},\hat{u}_{i}=\max_j r_{ij}$ are the MLEs of $l_i, u_i$.
 Let
 $\alpha\in [0,1]$ and $C\in \R^+$ be unknown problem-dependent constants.
 When $t_i\geq 2$,
 U/LCB1-Uniform has a worst-case polynomial, best-case constant cumulative regret upper bound per arm,
 which converges to $1+2C$ when $\alpha\to 1$:
 \begin{align*}
  \textstyle\frac{24(u_i-l_i)^2 (1-\alpha)^2 \log T}{\Delta_i^2} + 1 + 2C +  \frac{(1-\alpha)T(T+1)(2T+1)}{3}
  .
 \end{align*}
\end{theo}
\begin{proof}
 \emph{(Sketch of proof in appendix.)}
 We apply \emph{bounded difference inequality} \citep{boucheron2013concentration}
 to derive a confidence bound of $\frac{\hat{u}_i+\hat{l}_i}{2}$.
 It contains an unknown value $u_i-l_i$, which is an issue.
 Therefore,
 we use lemmas about
 the critical value $\alpha = P(r_i < X)$ of $r_i=\frac{\hat{u}_i-\hat{l}_i}{u_i-l_i}$,
 its CDF,
 and union-bound
 to derive a looser upper-bound.
 The coefficient 6 and $t_i\geq 2$ are derived from the condition that makes $C$ finite.
\end{proof}

Just like UCB1-Normal2, UCB1-Uniform is spread-aware.
The second term is scaled by the support range $\hat{u}_i-\hat{l}_i$,
similar to the empirical variance $\hat{\sigma}_i$ in UCB1-Normal2.
A larger spread indicates more chance that the next pull results in a wildly different, smaller $h$,
while a smaller spread indicates a plateau,
a region of flat $h$ landscape \citep{Coles07} that hinders the search progress, particularly when $\hat{u}_i-\hat{l}_i=0$.
Penalizing a small spread gives UCB1-Uniform/Normal2 an ability to avoid plateaus.

However,
UCB1-Uniform can not only avoid, but also escape plateaus quickly.
For example, in \refig{fig:plateau},
the two plateaus are equally informed ($u_1=u_2=6, l_1=l_2=4$) and $t_1>t_2$,
thus it keeps searching plateau 1 in a depth-first manner,
rather than distributing the effort and failing to explore either one sufficiently.

\begin{figure}[tb]
 \includegraphics{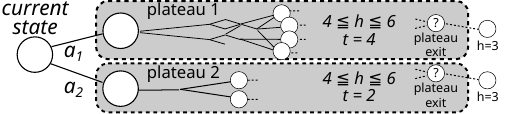}
 \caption{
 Given equally informative plateaus,
 UCB1-Uniform focuses on one plateau to find an exit quickly.}
 \label{fig:plateau}
\end{figure}

\section{Experimental Evaluation}
\label{sec:experiments}

We first evaluated the proposed algorithm implemented on Pyperplan \citep{pyperplan} by
counting the number of instances solved under 10,000 node evaluations
over
a subset of the International Planning Competition benchmark domains,
selected for compatibility with the set of PDDL extensions supported by Pyperplan
(772 problem instances across 24 domains in total).
We focus on node evaluations to improve the reproducibility by removing the effect of low-level implementation detail.
See the appendix for the results
controlled by expansions and the runtime.

We evaluated various algorithms with $\ff$, $\ad$, $\hmax$, and $\gc$ (goal count) heuristics \citep{FikesHN72},
and our analysis focuses on $\ff$.
We included $\gc$ despite its uninformativeness because it can be used in environments without domain descriptions,
e.g., in the planning-based approach \citep{lipovetzky2015classical} to the Atari environment \citep{bellemare2015arcade}.
We ran each configuration with 5 random seeds and report the average number of problem instances solved.
We do not evaluate \astar, \uct, and \uct* as we focus on agile search settings.
We included the evaluations with Deferred Evaluation (DE) and Preferred Operators (PO) (see appendix),
following \citet{schulte2014balancing}.

We then evaluated some of the algorithms reimplemented in Fast Downward \citep{Helmert2006}
on IPC2018 instances with $\ff$
under the agile IPC setting (5 minute, 8GB memory),
using Intel Xeon 6258R CPU @ 2.70GHz.

\begin{table*}[tb]
 \centering
\setlength{\tabcolsep}{0.3em}
\begin{adjustbox}{max width=0.7\linewidth}
  \begin{tabular}{l*{6}{c}}
  \toprule
  \multicolumn{1}{r}{$h=$}
  & \multicolumn{1}{c}{$\ff$}
  & \multicolumn{1}{c}{$\ad$}
  & \multicolumn{1}{c}{$\hmax$}
  & \multicolumn{1}{c}{$\gc$}
  & \multicolumn{1}{c}{$\ff$+PO}
  & \multicolumn{1}{c}{$\ff$+DE+PO}
  \\
\midrule
\gbfs(Pyperplan/FD) & 538/539 & 518/517 & 224/226 & 354/349 & $\dagger$/539  & $\dagger$/$\ddagger$ \\
Softmin-Type(h) & 576 & 542.6 & 297.2 & 357.6 & 575.8  & $\ddagger$ \\
\midrule
GUCT & 412 & 397.8 & 228.4 & 285.2 & 454.2 & 440.4 \\
GUCT* & 459.4 & 480.8 & 242.2 & 312.2 & 496.2 & 471.8 \\
GUCT-Normal & 283.4 & 265 & 212 & 233.4 & 372.4 & 381.6 \\
GUCT*-Normal & 318.8 & 300 & 215.2 & 246.2 & 378.05 & 386.9 \\
GUCT-Normal2 & 582.95 & 538 & 316.6 & 380.6 & 623.2 & 581.8 \\
GUCT*-Normal2 & 567.2 & 533.8 & 263 & 341.2 & 619.8 & 570.6 \\
\textbf{GUCT-Uniform (ours)} & \textbf{606.4} & \textbf{563.4} & \textbf{455.6} & \textbf{492.2} & \textbf{635.6} & \textbf{600.8} \\
CHK-Uniform & 375.4 & 338.8 & 224.8 & 296.6 & 454.8 & 458.2 \\
GUCT+-Normal2 & 578 & 550.4 & 442.4 & 490.6 & 630.6 & 582.2 \\
\midrule
MaxSearch & 253.75 & 243.4 & 260 & 255.2 & 368.6 & 355.6 \\
RobustUCT & 267.8 & 270.8 & 234 & 231.8 & 403 & 435.2 \\
ThresholdAscent & 162.4 & 163.8 & 170.4 & 164.4 & 165.8 & 172.2 \\
  \bottomrule
 \end{tabular}
\end{adjustbox}
 \hspace{1em}
\begin{adjustbox}{max width=0.23\linewidth}
\begin{tabular}{r|rrrrr}
\toprule
& domain    & \gbfs          & SM            &  N2       & Uni       \\ \midrule
\multirow{11}{*}{\rotatebox{90}{Instances solved}}
 & agricola  & 9.0           & 10.2         & 9.4           & \textbf{11.6} \\
 & caldera   & 4.0           & \textbf{7.0} & 6.4           & 5.8           \\
 & data-net  & 4.0           & \textbf{8.4} & 8.2           & 7.0           \\
 & flashfill & \textbf{9.0}  & 8.8          & 7.2           & 6.8           \\
 & nurikabe  & 7.0           & 6.2          & \textbf{8.4}  & 7.6           \\
 & org-syn   & 9.0           & 8.8          & \textbf{9.2}  & 6.0           \\
 & settlers  &               & \textbf{5.4} & 2.4           & 2.6           \\
 & snake     & 5.0           & 5.0          & 15.4          & \textbf{19.0} \\
 & spider    & 8.0           & 8.2          & \textbf{9.2}  & 8.6           \\
 & termes    & \textbf{12.0} & 11.6         & 5.8           & 5.0           \\   \cmidrule{2-6}
 & total     & 67.0          & 79.6         & \textbf{81.6} & 80.0          \\ \midrule
\multirow{11}{*}{\rotatebox{90}{IPC score}}
 & agricola  & 1.9          & 3.4          & 2.0          & \textbf{6.1}  \\
 & caldera   & 2.9          & 5.0          & 5.1          & \textbf{5.3}  \\
 & data-net  & 3.5          & 4.5          & \textbf{5.6} & 4.9           \\
 & flashfill & 6.2          & \textbf{6.8} & 5.4          & 4.8           \\
 & nurikabe  & 6.4          & 5.4          & \textbf{6.9} & 6.5           \\
 & org-syn   & \textbf{7.2} & 6.8          & 6.7          & 5.0           \\
 & settlers  &              & \textbf{2.6} & 1.6          & 2.3           \\
 & snake     & 2.9          & 3.3          & 9.1          & \textbf{12.5} \\
 & spider    & 2.2          & 3.1          & \textbf{3.4} & \textbf{3.4}  \\
 & termes    & \textbf{6.4} & \textbf{6.2} & 2.6          & 2.5           \\\cmidrule{2-6}
 & total     & 39.5         & 47.0         & 48.6         & \textbf{53.2} \\\bottomrule
\end{tabular}
\end{adjustbox}
\caption{
 Best algorithms in \textbf{bold}.
 \textbf{(left)}
 The number of problem instances solved with less than 10,000 node evaluations;
 each number represents an average over 5 seeds.
 PO/DE stand for Preferred Operators/Deferred Evaluation.
 Pyperplan supports PO only for $\ff$.
 $\dagger$: Data missing due to the lack of support of PO for \gbfs in Pyperplan.
 $\ddagger$: Data missing because DE in Fast Downward measures node evaluations differently.
 \textbf{(right)}
 Number of instances solved and IPC scores on IPC 2018 instances,
 using $\ff$ under 5 min time limit and 8GB memory limit, averaged over 3 seeds.
 For caldera and organic-synthesis,
 we used their action-splitting variants \citep{areces2014optimizing}
 provided by the organizers.
 `SM' stands for Softmin-Type(h),
 `N2' stands for \guct-Normal2,
 `Uni' stands for \guct-Uniform.
 }
 \label{tbl:main-table}
\end{table*}

\paragraph{Queue-based}

We first evaluated \lsota queue-based search algorithms
and compared them against our proposed \guct-Uniform.
In \reftbl{tbl:main-table},
\gbfs (Pyperplan/FD) shows the results of \gbfs implemented in Pyperplan and FastDownward, respectively.
We evaluated them both to confirm the effect of implementation difference.
We next evaluated Softmin-Type(h) \citep{kuroiwa2022biased},
a recent \lsota diversified search algorithm for classical planning,
from the original C++ implementation available online.
GUCT-Uniform outperformed
\gbfs and Softmin-Type(h) by 67.8 and 33.2 instances, respectively.

\paragraph{$*$-Variants}

We next compared various bandit algorithms with Monte Carlo backup and Full Bellman backup
to analyze the effect of backup differences.
In \reftbl{tbl:main-table},
\guct is a \guct that uses the original UCB1 bandit for action selection.
Note that
this does not have the ``normalization'' feature \citep{schulte2014balancing}
that turned out to be harmful \citep{wissow2024scale}.
\guct-Normal uses UCB1-Normal \citep{auer2002finite}, and
\guct-Normal2 uses UCB1-Normal2 \citep{wissow2024scale}.
The $*$-variants (\guct*-Normal, etc.) use Full Bellman backup instead of Monte Carlo backup.

While $*$-variants tend to improve the performance over the base Monte Carlo variants,
it happens only when the base algorithm is non-performant.
\guct*-Normal2 performs significantly worse than \guct-Normal2,
and \guct/* and \guct/*-Normal are vastly inferior to \guct/*-Normal2.
Our proposed GUCT-Uniform outperformed both $*$- and base variants:
GUCT*, *-Normal, *-Normal2, GUCT, -Normal, and -Normal2 by +194.4, +323, +23.4, +147, +287.6, and +39.2 respectively.

These results demonstrate the benefit of
selecting a backup method that is theoretically consistent with the given bandit.
The Full Bellman Backup estimates
Uniform distributions with unknown support,
and negatively affects the performance of \guct*-Normal2's Gaussian bandit
with its conflicting assumptions.
UCB1-Uniform, which does not have such theoretical dissonance,
managed to extract the best performance while using Full Bellman backup.

\paragraph{CHK-Uniform}

We added
CHK-Uniform \citep{cowan2015asymptotically},
an asymptotically optimal bandit
for
Uniform distributions.
To our knowledge,
CHK-Uniform is the only bandit that works on uniform distributions with unknown supports
and is asymptotically optimal,
providing a baseline for UCB1-Uniform.
Our UCB1-Uniform significantly outperformed CHK-Uniform.
This is another interesting case of a non-asymptotically-optimal bandit outperforming an asymptotically optimal one,
such as UCB1-Uniform vs. CHK-Uniform and UCB1-Normal2 vs. UCB1-Normal.
A deeper theoretical investigation into this phenomena is an important avenue of future work.

\paragraph{+-Variants}

GUCT$^{+}$-Normal2 
uses the backup method explained in \refsec{sec:issues-minimum}
that estimates the
maximum of Gaussian RVs,
modified for minimization.
While this backup sometimes improved the results from \guct/*-Normal2,
the improvement depends on the heuristics and they were overall outperformed by GUCT-Uniform.
The likely explanation is that the Maximum-of-Gaussians method is not accurate
for combining the estimates for more than 2 arms.

\paragraph{Max-$k$ Bandits}

Our work can be confused with 
the \emph{Max $k$-Armed Bandit} framework
\citep{cicirello2004heuristic, cicirello2005max, streeter2006asymptotically, streeter2006simple, carpentier2014extreme, achab2017max}
that optimizes \emph{extreme regret} $\max_i \E[\max_{t=1}^T \rr_{it}] - \E[\max_{t=1}^T \rr_{I_tt}]$,
where $I_t$ is the arm pulled at $t$.
While both approaches use EVTs,
UCB1-Uniform is not a Max-$k$ bandit algorithm and has many theoretical/practical/conceptual differences.
Existing Max-$k$ bandits
target long-tail distributions while we target short-tail distributions,
they primarily use block maxima EVTs, and
they fail to align conceptually with classical planning
(due to space,
this discussion continues in the appendix).
We focus here on the experimental results.

We evaluated GUCT variants that use three Max-$k$ bandits for action selection:
\emph{Threshold Ascent} \citep{streeter2006simple}, 
\emph{RobustUCB} \citep{bubeck2013bandits}, 
\emph{MaxSearch} \citep{kikkawa2022materials}. 
All hyperparameters are based on the values suggested by their authors.
\reftbl{tbl:main-table} shows that
these algorithms significantly underperformed.

\paragraph{Agile Experiments with C++}

\reftbl{tbl:main-table} (right) shows the results comparing C++ implementations of \guct-Uniform and other algorithms on Fast Downward.
UCB1-Uniform was on par with Softmin-Type(h) and \guct-Normal2 in terms of the number of solved instances,
and outperform them on the IPC score $\sum_i \min (1, 1- \frac{\log t_i}{\log 300})$,
where $t_i$ is the runtime of each algorithm solving an instance $i$.

\section{Conclusion}
\label{sec:conclusion}

Previously, statistical estimates for guiding MCTS in classical planning
did not respect the natural properties of heuristic functions, i.e.\ that they have unknown, half-bounded support,
leading to overspecification (UCB1: known finite support) or underspecification (Gaussian bandit: entire $\R$).
Also,
why Monte Carlo backup (averaging)
was used
for minimization/maximization tasks
was unclear.
In searching for a theoretically justified backup for agile planning,
we modeled the rewards
with Peaks-Over-Threshold Extreme Value Theory (POT EVT), which captures the finer details of heuristic search.
This led to our new bandit, UCB1-Uniform, which uses the MLE of the Uniform distribution
to guide action selection.
The resulting algorithm outperformed
\gbfs,
\guct-Normal2,
asymptotically optimal uniform bandit CHK-Uniform,
a \lsota diversified search algorithm Softmin-Type(h),
and Max-$k$ bandits.

\section*{Acknowledgments}

This work was supported through DTIC contract FA8075-18-D-0008, Task
Order FA807520F0060, Task 4 - Autonomous Defensive Cyber Operations
(DCO) Research \& Development (R\&D).

\appendix

\fontsize{9.5pt}{10.5pt}
\selectfont

\end{document}